\documentclass{article} % For LaTeX2e
\usepackage[preprint]{colm2026_conference}

\usepackage{microtype}
\usepackage{hyperref}
\usepackage{url}
\usepackage{booktabs}
\usepackage{amsthm}
\usepackage{amssymb}
\usepackage{algorithmic}
\usepackage{booktabs, algorithm, algorithmic, xcolor}
\usepackage{amsmath}
\usepackage{multirow}
\usepackage{tcolorbox}
\usepackage{tikz}
\usetikzlibrary{positioning, shapes, arrows.meta}
\usepackage{pgfplots}
\pgfplotsset{compat=1.18}
\usetikzlibrary{backgrounds}
\usepackage{wrapfig}
\usepackage{enumitem}
\usepackage{caption}
\usepackage{cleveref}
\usepackage{listings}
\usepackage{xcolor}
\usepackage{textcomp}
\usepackage{tcolorbox}

\tcbuselibrary{skins}
\newtcolorbox{importantbox}{
  enhanced,
  colback=cyan!10,
  colframe=black!60,
  boxrule=1pt,
  arc=2pt,
  width=\linewidth,
  center,
  left=6pt,
  right=6pt,
  top=2pt,
  bottom=2pt
}

\definecolor{genblue}{RGB}{60,100,200}
\definecolor{highlight}{RGB}{255,235,160}
\definecolor{goodgreen}{RGB}{46,125,50}
\definecolor{badred}{RGB}{200,50,50}

\usetikzlibrary{tikzmark}

\newlength{\tokenwidth}
\newtheorem{theorem}{Theorem}
\theoremstyle{definition} % Optional: formats the body in normal font
\newtheorem{definition}{Definition}[section]
\newtheorem{proposition}{Proposition}
\newtheorem{remark}{Remark}

\usepackage{lineno}

\definecolor{darkblue}{rgb}{0, 0, 0.5}
\hypersetup{colorlinks=true, citecolor=darkblue, linkcolor=darkblue, urlcolor=darkblue}

\title{Adaptive Block Diffusion: Resolving Training–Inference \\Mismatch in Diffusion Language Models}

\author{Gagan Jain \thanks{Work done while at Microsoft AI. Email: \texttt{gaganjain1582@gmail.com}}\\
}

\begin{document}

\ifcolmsubmission
\linenumbers
\fi

\maketitle

\begin{abstract}
Diffusion Language Models (DLMs) are typically trained under fixed context structures, restricting denoising to predetermined token subsets. This creates a mismatch between training and inference, where models must operate over arbitrary configurations, leading to degradation off the training grid. We propose \textbf{Adaptive Block Diffusion (ABD)}, which resolves this mismatch by optimizing denoising risk over a distribution of prefix–window configurations. By treating the configuration as a stochastic variable, ABD trains a single model over the full configuration space without architectural changes. We show that generalization across decoding strategies is governed by the support of the training distribution, and that ABD guarantees denoising optimality for any inference policy whose configurations are covered during training. Empirically, ABD exhibits structural invariance across decoding scales, avoiding off-grid collapse and recovering a monotonic relationship between block size and perplexity, while matching or outperforming fixed-block specialists at their target scales.

\end{abstract}

\section{Introduction}
Autoregressive (AR) language models \citep{Radford2018ImprovingLU} generate text sequentially and achieve strong perplexity, but their inference latency scales linearly with sequence length, limiting parallelism. Diffusion-based approaches, originally developed for continuous domains 
\begin{wrapfigure}{r}{0.5\textwidth}
        \centering
        \vspace{-12pt} % Pull figure up to align with text
        \includegraphics[width=0.48\textwidth]{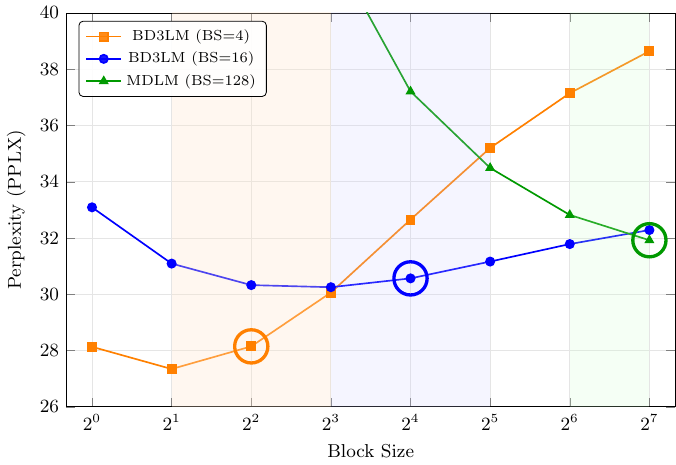}
        \captionof{figure}{\scriptsize \textbf{The Generalization Gap in Fixed-Block Diffusion:} Eval Perplexity as a function of inference block size on LM1B for models trained at fixed block sizes. Fixed-block specialists exhibit optimal performance only near their training support (circled), with sharp perplexity spikes occurring when evaluated at off-grid scales, demonstrating the lack of structural invariance, where models fail to generalize to configurations unseen during training.}
        \label{fig:blocksize-pplx}
        \vspace{-15pt} % Reduce space below figure
    \end{wrapfigure}
\citep{sohldickstein2015deepunsupervisedlearningusing, ho2020denoisingdiffusionprobabilisticmodels, song2022denoisingdiffusionimplicitmodels, dhariwal2021diffusionmodelsbeatgans}, have recently been adapted to discrete text generation \citep{austin2023structureddenoisingdiffusionmodels, hoogeboom2021argmaxflowsmultinomialdiffusion}, offering a complementary paradigm that enables parallel decoding.

To bridge the gap between global parallelism and local precision, recent work has introduced structured diffusion schemes such as Masked Discrete Diffusion Language Models (MDLM) \citep{sahoo2024simpleeffectivemaskeddiffusion} and Block Diffusion \citep{arriola2025blockdiffusioninterpolatingautoregressive}, which restrict denoising to subsets of prefix conditioned tokens. However, these approaches rely on \textit{fixed context structures}, where the set of denoised tokens follows a predetermined pattern throughout training.

This design introduces a fundamental limitation: the model is trained to denoise under a narrow set of configurations, while inference often requires operating over arbitrary prefix lengths and window sizes. This induces a \textit{support mismatch} between training and inference. When evaluated outside their training support, fixed-block models exhibit sharp degradation in perplexity and unstable generation behavior, as shown in \autoref{fig:blocksize-pplx}. These models do not learn a general denoising procedure, but rather specialize to configurations seen during training, leaving their behavior on unseen configurations effectively unconstrained. This motivates us to raise a fundamental question:

\begin{figure*}[t]
  \centering
  \includegraphics[width=\linewidth]{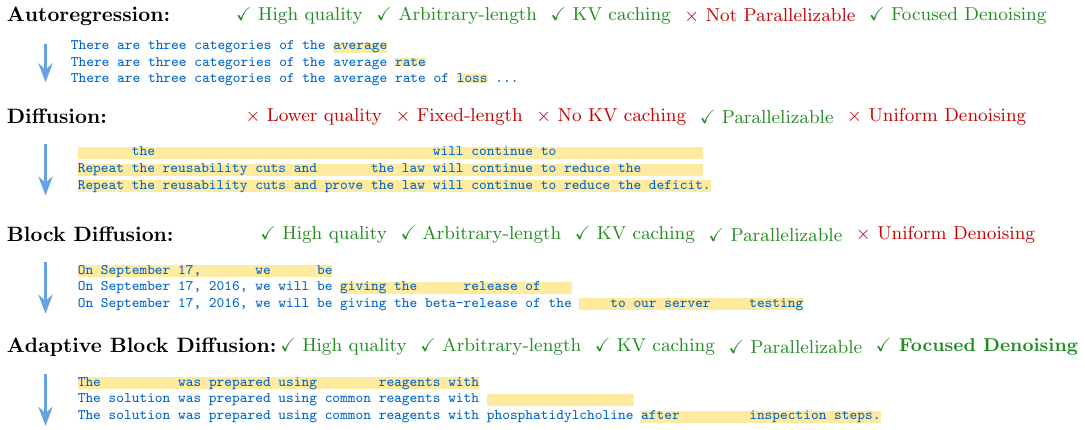}
  \caption{\textbf{Comparison of Text Generation Paradigms:} Autoregressive models generate tokens sequentially, while Standard Diffusion treats sequences globally. Block Diffusion introduces structure but remains restricted to fixed-length, grid-aligned windows. Adaptive Block Diffusion generalizes these frameworks by training on a stochastic distribution of prefix lengths and window sizes, enabling high-quality, parallel generation across arbitrary boundaries while maintaining KV caching compatibility.
  }
  \label{fig:abd-overview}
  \vspace{-17pt}
\end{figure*}

\begin{importantbox}
\vspace{-2pt}
\centering \textit{Can a \textbf{single diffusion language model} be trained to eliminate the mismatch between\\ training and inference, remaining invariant to decoding configurations?}
\vspace{-2pt}
\end{importantbox}

In this paper, we propose \textbf{Adaptive Block Diffusion (ABD)}, a framework that resolves this mismatch by training over the full configuration space of prefix–window pairs. Rather than restricting training to a fixed set of denoising patterns, ABD treats the configuration as a stochastic variable and optimizes denoising risk across a distribution of configurations. Unlike inference-time approaches \citep{lu2026adablockdllmsemanticawarediffusionllm}, which adapt block size post-hoc without training-time guarantees, ABD ensures that the support of inference-time configurations is covered during training, yielding a single model that is \textit{structurally invariant} to decoding scale and boundary choices. Our contributions are as follows:
\vspace{-4pt}
\begin{itemize}[leftmargin=*]
    \setlength\itemsep{1pt}
    
    \item \textbf{The Adaptive Block Diffusion Framework:} We introduce a training framework that treats denoising configurations as stochastic variables, enabling optimization over the full configuration space of prefix–window pairs.

    \item \textbf{Training–Inference Alignment:} We show that ABD guarantees denoising optimality for any inference policy whose configuration distribution is supported during training, formalizing alignment via a Radon–Nikodym argument.

    \item \textbf{Structural Invariance:} We demonstrate that ABD eliminates off-grid degradation and recovers a monotonic relationship between block size and perplexity, in contrast to fixed-block specialists.

    \item \textbf{Unified Model Across Scales:} A single ABD model matches or outperforms fixed-block models at their target configurations while maintaining strong performance across all decoding scales.
\end{itemize}

\section{Preliminaries}

\textbf{Masked Discrete Diffusion Language Models}: Let $\mathcal{V}$ be a finite vocabulary, $\mathcal{D}$ denote the training data
distribution, and $\mathbf{x}_t = (x_t^1,\dots,x_t^L) \in \mathcal{V}^L$ denote a
length-$L$ sequence of discrete tokens at diffusion timestep $t$. Masked discrete
diffusion language models (MDLMs) define a Markov forward process $q$ where
$\mathbf{x}_t \in (\mathcal{V}~\cup~\{\text{MASK}\})^L$ denotes a corrupted version
of $\mathbf{x}_0$ at timestep $t$. The forward transition kernels are typically
chosen to factorize across token positions, with corruption strength increasing
in $t$, $q(\mathbf{x}_{1:T} \mid \mathbf{x}_0) =
\prod_{t=1}^T q_t(\mathbf{x}_t \mid \mathbf{x}_{t-1}) =
\prod_{t=1}^T \prod_{i=1}^L q_t(x_t^i \mid x_{t-1}^i).
$
The conditional distribution $q_t(\cdot \mid \mathbf{x}_0)$ is assumed to have
full support over the masked token for all $t$. The reverse process is
parameterized by a neural model $p_\theta$, which defines conditional
distributions of the form $p_\theta(\mathbf{x}_0 \mid \mathbf{x}_t, t)$. Training minimizes an expected denoising objective that can be written as a sum
of Kullback--Leibler divergences between the true reverse conditionals induced by
the forward process and the model distributions:
\[
\mathbb{E}_{\mathbf{x}_0 \sim \mathcal{D},\;
t \sim \mathcal{U}(\{1,\dots,T\}),\;
\mathbf{x}_t \sim q_t(\cdot \mid \mathbf{x}_0)}
\sum_{i=1}^L
\mathrm{KL}\!\left(
q(x_0^i \mid x_t^i, t)
\;\middle\|\;
p_\theta(x_0^i \mid x_t^i, t)
\right)
\]
where $q(x_0^i \mid x_t^i, t)$ denotes the true reverse conditional associated
with the forward diffusion process. This objective is equivalent to minimizing
the denoising cross-entropy loss up to an additive constant independent of
$\theta$. A key property of this formulation is that, for any fixed diffusion timestep
$t$, denoising is performed in parallel over a fixed set of token positions.
While token values are iteratively refined across timesteps, the index set of
positions being denoised remains unchanged within a diffusion run.

\textbf{Configuration Space Notation}: To formalize the space of denoising configurations, we introduce a prefix length $k$ and a denoising window length $\ell$, and define the \emph{configuration space} as
$\Omega := \{(k,\ell) \mid 0 \le k \le L-1,\; 1 \le \ell \le L-k\}$. For any configuration $(k,\ell) \in \Omega$, we define the set of noised positions as $S(k,\ell) := \{k+1,\dots,k+\ell\}$. A binary noise mask is defined as $m_{k,\ell}(i) := \mathbf{1}\!\left\{ i \in S(k,\ell) \right\},  i=1,\dots,L$, where $\mathbf{1}\{\cdot\}$ denotes the indicator function. This formulation makes explicit that each diffusion objective corresponds to optimizing denoising performance over a subset of the configuration space $\Omega$ during training. 

Given a configuration $(k,\ell)$, tokens indexed by $S(k,\ell)$ are corrupted by the forward diffusion process, while tokens outside this set remain clean and serve as conditioning context. The diffusion denoising process is therefore fully specified by the tuple $(x_0, k, \ell, t)$. This formulation subsumes a broad class of structured diffusion procedures in
which denoising is restricted to a contiguous suffix conditioned on a clean
prefix. Specific instantiations correspond to imposing additional constraints
on the configuration space $\Omega$.

\textbf{Block Diffusion and Semi-Autoregressive Decoding}: Block diffusion corresponds to a particular restriction of the configuration
space defined above. Specifically, training is performed on configurations of the form $\Omega_{\mathrm{BD}}=
\{(k,\ell) \in \Omega \mid \ell = L',\; k \in bL', b \in \mathbb{Z}^+\}$, where the denoising window length is fixed to a block size $L'$ and the prefix
length advances in increments of $L'$. Each diffusion run denoises a fixed-length window following a block-aligned prefix. As a result, the model is only trained to denoise under configurations in $\Omega_{\mathrm{BD}}$, leaving its behavior on the remainder of $\Omega$ unconstrained by the training objective. This perspective motivates training objectives that optimize denoising risk over broader subsets of $\Omega$, which we develop in the following section as Adaptive Block Diffusion.

\section{Methodology: Adaptive Block Diffusion}

\label{sec:abd-obj}
We now introduce \emph{Adaptive Block Diffusion} (ABD), which resolves the training–inference mismatch by optimizing denoising risk over the full configuration space of prefix-window pairs. Rather than restricting training to a fixed block-aligned subset, ABD treats the configuration $(k,\ell) \in \Omega$ as a stochastic variable and directly optimizes denoising risk over this space. To this end, we define a training distribution $\pi$ over the configuration space $\Omega$. Given a configuration $(k,\ell) \sim \pi$, diffusion is applied only to tokens indexed by $S(k,\ell)$, while tokens outside this set remain clean and serve as conditioning context.

Formally, the ABD training objective is defined as
\[
\mathcal{L}_{\mathrm{ABD}}(\theta)
=
\mathbb{E}_{{{(k,\ell)\sim\pi}},\;x_0 \sim \mathcal{D},\;
t \sim \mathcal{U}(\{1,\dots,T\}),\;
\mathbf{x}_t \sim q_t(\cdot \mid \mathbf{x}_0)}
\sum_{{{i \in S(k,\ell)}}}
\mathrm{KL}\!\left(
q(x_0^i \mid x_t^i, t)
\;\middle\|\;
p_\theta(x_0^i \mid x_t^i, \mathbf{x}_0^{1:k}, t)
\right)
\]

This objective can be interpreted as minimizing expected denoising risk over the configuration space $\Omega$ during training under the distribution $\pi$. Thus, existing diffusion objectives can be viewed as special cases of $\mathcal{L}_{\mathrm{ABD}}$, corresponding to restricted choices of $\pi$ with limited support over $\Omega$. This objective reduces to the standard masked diffusion objective when
$\ell = L$ and to block diffusion when $\pi$ is supported only on configurations
with fixed window length $\ell = L'$ and block-aligned prefix lengths
$k \in \{0, L', 2L', \dots\}$. Importantly, ABD does not require any architectural
changes to the diffusion model; it modifies only the distribution over
configurations encountered during training. 

By training over configurations with varying prefix lengths and denoising
window sizes, ABD exposes the model to the full range of contexts that may arise
under adaptive or sliding-window decoding strategies. This enables the model to learn denoising behavior that generalizes across arbitrary inference-time policies that select $(k,\ell)$ dynamically. Any distribution $\pi$ with support over $\Omega$ is valid in principle. In our experiments, we consider simple choices that factorize over prefix length and
window length, sampling $k$ and $\ell$ independently. Algorithm~\ref{alg:abd-training} summarizes the resulting training procedure, which is identical to standard diffusion except for the distribution over configurations.

Importantly, the theoretical properties of ABD derived in the following section
do not depend on a specific choice of $\pi$, but rather on its support. As long as $\pi$ assigns nonzero probability to configurations encountered at inference time, the corresponding denoising risks are optimized during training, ensuring that inference-time configurations are not out-of-distribution.

\begin{algorithm}[t]
\caption{Adaptive Block Diffusion Training}
\label{alg:abd-training}
{\small\begin{algorithmic}[1]
\REQUIRE Training data dist. $\mathcal{D}$, block config dist. $\pi$ over $(k, \ell)$ pairs, model parameters $\theta$
\vspace{2pt}
\hrule height 0.3pt
\vspace{4pt}
\WHILE{not converged}
    \STATE Sample a training sequence $\mathbf{x}_0 \sim \mathcal{D}$
    \STATE Sample a block configuration $(k,\ell) \sim \pi$ defining the active block $S(k,\ell)$
    \STATE Sample a diffusion timestep $t \sim \mathcal{U}(\{1,\dots,T\})$ and corrupt the sequence to obtain $\mathbf{x}_t \sim q_t(\cdot \mid \mathbf{x}_0)$
    \vspace{-9pt}
    \STATE Compute ABD loss, denoising positions in active block while conditioning on prefix $\mathbf{x}_0^{1:k}$:
    \vspace{-5pt}
    \[
    \mathcal{L}_{\mathrm{ABD}}(\theta)
    =
    \sum_{i \in S(k,\ell)}
    \mathrm{KL}\!\left(
    q(x_0^i \mid x_t^i, t)
    \;\middle\|\;
    p_\theta(x_0^i \mid x_t^i, \mathbf{x}_0^{1:k}, t)
    \right)
    \]
    \vspace{-10pt}
    \STATE Update $\theta$ by taking a gradient step on $\mathcal{L}_{\mathrm{ABD}}(\theta)$
\ENDWHILE
\end{algorithmic}}
\end{algorithm}

\section{Theoretical Analysis}

We analyze the Adaptive Block Diffusion (ABD) objective introduced in Section~\ref{sec:abd-obj} and establish formal guarantees regarding denoising optimality and training--inference alignment. Our analysis treats diffusion as a conditional risk minimization problem over the configuration space $\Omega$, revealing that the generalization behavior of diffusion models is governed by the support of the training distribution over configurations.

\subsection{Risk over Configuration Space}

We formalize diffusion training as risk minimization over the configuration space $\Omega$. For a given configuration $(k,\ell) \in \Omega$, recall that denoising is performed over the index set $S(k,\ell) = \{k+1,\dots,k+\ell\}$, conditioned on
the clean prefix $\mathbf{x}_0^{1:k}$. We define the conditional denoising risk
of a model $p_\theta$ at configuration $(k,\ell)$ as
\[
\mathcal{R}(\theta; k,\ell)
=
\mathbb{E}_{\mathbf{x}_0 \sim \mathcal{D},\;
t \sim \mathcal{U}(\{1,\dots,T\}),\;
\mathbf{x}_t \sim q_t(\cdot \mid \mathbf{x}_0)}
\Bigg[
\sum_{i \in S(k,\ell)}
\mathrm{KL}\!\left(
q(x_0^i \mid x_t^i, t)
\;\middle\|\;
p_\theta(x_0^i \mid x_t^i, \mathbf{x}_0^{1:k}, t)
\right)
\Bigg].
\]

Given a training distribution $\pi$ over the configuration space $\Omega$, the
expected denoising risk is defined as
\[
\mathcal{R}_\pi(\theta)
=
\mathbb{E}_{(k,\ell)\sim\pi}\big[
\mathcal{R}(\theta; k,\ell)
\big].
\]

By construction, minimizing $\mathcal{R}_\pi(\theta)$ is equivalent to minimizing
the Adaptive Block Diffusion objective defined in Section~\ref{sec:abd-obj}. Thus, training a
diffusion model corresponds to minimizing expected denoising risk over configurations
sampled from $\pi$, making the support of $\pi$ the key factor governing
generalization across configurations.

\subsection{Consistency of Adaptive Block Diffusion}

We first establish that ABD is statistically consistent over the support of the
training configuration distribution.

\begin{theorem}[Optimality of ABD over Configuration Space]
\label{thm:abd-optimality}
Let $\mathcal H = \{p_\theta : \theta \in \Theta\}$ be a hypothesis class of
conditional distributions. Then any global minimizer $\theta^\star \in \arg\min_{\theta \in \Theta} \mathcal{R}_\pi(\theta)$ satisfies
\[
\mathcal{R}(\theta^\star; k,\ell)
=
\inf_{\theta \in \Theta}
\mathcal{R}(\theta; k,\ell)
\quad \text{for $\pi$-almost every $(k,\ell) \in \Omega$.}
\]
\end{theorem}

\noindent
This result shows that ABD achieves optimal denoising performance on all configurations within the support of the training distribution. Configurations outside this support are not constrained by the objective, highlighting the central role of support coverage in determining generalization.

\subsection{Training--Inference Alignment}

We now formalize inference-time configuration policies and show that ABD
guarantees denoising optimality under arbitrary such policies, provided their
support is covered during training.

\begin{definition}[Inference Policy]
An inference policy induces a (possibly stochastic) sequence of configurations
$\{(k_t,\ell_t)\}_{t \ge 0}$ during generation. Let
$\Pi_{\mathrm{inf}}$ denote the distribution over configurations induced by this
policy.
\end{definition}

\begin{theorem}[Training--Inference Alignment]
\label{thm:alignment}
Let $\pi$ be the training configuration distribution and
$\Pi_{\mathrm{inf}}$ be the configuration distribution induced by an inference
policy. Suppose that $\Pi_{\mathrm{inf}} \ll \pi$.

Then, over the same hypothesis class $\mathcal{H}$, any global minimizer $\theta^\star \in \arg\min_{\theta \in \Theta} \mathcal{R}_\pi(\theta)$ is also a global minimizer of the inference risk
\[
\mathcal{R}_{\mathrm{inf}}(\theta)
=
\mathbb{E}_{(k,\ell)\sim\Pi_{\mathrm{inf}}}
\big[
\mathcal{R}(\theta; k,\ell)
\big]
\]
\end{theorem}

\noindent
Theorem~\ref{thm:alignment} shows that as long as the support of the inference
distribution is contained within the support of the training distribution,
denoising optimality transfers from training to inference. This result shows that generalization across decoding configurations is not a property of the model architecture, but of the support of the training distribution over $\Omega$.

\begin{remark}[Support mismatch and density effects]
If $\Pi_{\mathrm{inf}} \not\ll \pi$, the inference risk decomposes into
in-distribution and out-of-distribution components, where the latter corresponds
to configurations unseen during training and is unconstrained by the objective.
When $\Pi_{\mathrm{inf}} \ll \pi$, the inference risk can be written as a
reweighted expectation under $\pi$, where large density ratios amplify errors in
configurations that are underrepresented during training. This highlights the
importance of balanced configuration coverage.
\end{remark}

\subsection{Limitation of Block Diffusion}

We contrast the above guarantees with those of standard block diffusion, which
restricts training to a strict subset of the configuration space $\Omega$.

\begin{proposition}[Limitation of Block Diffusion]
\label{prop:block-limitation}
Let $\pi_{\mathrm{BD}}$ be the training configuration distribution for block
diffusion, supported on a strict subset
$\Omega_{\mathrm{BD}} \subsetneq \Omega$. Let
$\theta^\star \in \arg\min_{\theta \in \Theta} \mathcal{R}_{\pi_{\mathrm{BD}}}(\theta)$ be a global minimizer of the block diffusion objective.

Then, for any configuration $(k,\ell) \notin \Omega_{\mathrm{BD}}$, the value of
$\mathcal{R}(\theta^\star; k,\ell)$ is unconstrained by the training objective.
Consequently, if $\Pi_{\mathrm{inf}} \not\ll \pi_{\mathrm{BD}}$, block diffusion
provides no guarantee of denoising optimality under the inference policy.
\end{proposition}

\noindent
Proposition~\ref{prop:block-limitation} formalizes the failure mode of fixed-block
diffusion: by restricting training to $\Omega_{\mathrm{BD}}$, these models leave
denoising behavior on unseen configurations unconstrained. When inference
procedures induce configurations outside this support, performance can degrade significantly, explaining the empirical instability observed off the training grid.

\begin{remark}
The condition $\Pi_{\mathrm{inf}} \ll \pi_{\mathrm{BD}}$ requires that inference
induce only block-aligned configurations with fixed window length. In practice,
common decoding strategies, including sliding-window generation, streaming
decoding, and adaptive window schedules, violate this condition, motivating the
use of training objectives with broader configuration support.
\end{remark}

\section{Experiments}

We evaluate ABD on two standard language modeling benchmarks: the One Billion Word corpus (LM1B; \cite{chelba2014billionwordbenchmarkmeasuring}) and OpenWebText (OWT; \cite{Gokaslan2019OpenWeb}). We use a transformer-based architecture of identical size and configuration to MDLM \citep{sahoo2024simpleeffectivemaskeddiffusion} and BD3LM \citep{arriola2025blockdiffusioninterpolatingautoregressive}, which share the same underlying architecture. This ensures that differences in performance are attributable to training objective and configuration distribution rather than model capacity. Our primary ABD model is initialized from a checkpoint pretrained with the full sequence length MDLM objective for 250k gradient steps, followed by finetuning for 750k steps under the categorical exponential configuration distribution, with probability proportional to $2^{-k}$ for the $k^{th}$ block size in $\{1, 2, 4, \dots 128\}$, normalized to sum to 1. We compare against autoregressive (AR), MDLM, SEDD, and BD3LM at block sizes 4 and 16 as baselines. Our experiments are designed to test the central theoretical prediction: that generalization across decoding configurations is governed by the support of the training distribution over $\Omega$.

\subsection{Structural Invariance}

Table~\ref{tab:main_results} reports perplexity across all inference block sizes for all models on LM1B and OWT. The results reveal a direct consequence of training on restricted configuration support. Fixed-block models optimize denoising performance only on a narrow subset of $\Omega$, and therefore fail to generalize to configurations outside their training support.

\begin{table}[b]
\centering
\setlength{\tabcolsep}{3pt}
\caption{Perplexity across inference block sizes on LM1B and OWT. BD3LM specialists exhibit sharp degradation at off-grid configurations. In parity with BD3LM, ABD is trained on 65B tokens on LM1B and 524B tokens on OWT. \textbf{Bold}: best result per column, \underline{underline}: second best.}
\vspace{-5pt}
\label{tab:main_results}
\footnotesize
\begin{tabular}{lcccccccc|cccc}
\toprule
& \multicolumn{8}{c|}{\textbf{LM1B}} & \multicolumn{4}{c}{\textbf{OWT}} \\
\cmidrule(lr){2-9} \cmidrule(lr){10-13}
\textbf{Model} & \textbf{1} & \textbf{2} & \textbf{4} & \textbf{8} & \textbf{16} & \textbf{32} & \textbf{64} & \textbf{128} & \textbf{1} & \textbf{4} & \textbf{16} & \textbf{1024} \\
\midrule
AR        & \textbf{22.88} & - & - & - & - & - & - & - & \textbf{17.54} & - & -  & -  \\
SEDD      & - & - & - & - & - & - & - & 32.68 & - & - & -  & \underline{24.10} \\
MDLM      & 106.56 & 74.16 & 52.48 & 42.35 & 37.21 & 34.49 & \underline{32.82} & \textbf{31.92} & - & - & - & \textbf{22.98} \\
BD3LM (BS=4)  & 28.13 & \underline{27.34} & \underline{28.15} & \underline{30.06} & 32.65 & 35.20 & 37.15 & 38.64 & - & \textbf{20.73} & - & - \\
BD3LM (BS=16) & 33.09 & 31.09 & 30.33 & 30.25 & \textbf{30.56} & \textbf{31.16} & \textbf{31.78} & \underline{32.28} &  - & - & \textbf{22.27} & -\\
\midrule
ABD (ours) & \underline{24.76} & \textbf{25.86} & \textbf{27.50} & \textbf{29.30} & \underline{31.04} & \underline{32.46} & 33.31 & 33.73 & \underline{19.57} & \underline{21.03} & \underline{22.56} & 25.48 \\
\bottomrule
\end{tabular}
\end{table}

A well-behaved generative model should exhibit monotonically improving perplexity as block size decreases, reflecting increasing conditional precision. This monotonic relationship is a structural property of the underlying language modeling task, and should hold independently of the training configuration. Smaller blocks denoise fewer tokens in parallel, approaching the precision of autoregressive generation. Fixed-block specialists violate this structure outside their training support, collapsing in both directions: BD3LM degrades sharply at block sizes larger than its training configuration, while MDLM collapses at small block sizes. Proposition~\ref{prop:block-limitation} explains this: configurations not seen during training are unconstrained by the objective, leading to degraded performance when evaluated off-grid.

ABD, by contrast, recovers the correct monotonic relationship across the full configuration range, without any off-grid degradation. This behavior is consistent with Theorem~\ref{thm:alignment}: by training over a distribution with broad support, ABD ensures that all inference-time configurations are optimized during training. Crucially, this structural invariance does not come at the cost of per-scale quality: ABD matches or outperforms fixed-block specialists at their home configurations, demonstrating that multi-scale training regularizes rather than compromises performance. On OWT, ABD remains within 0.3 perplexity points of specialists at their home configurations while maintaining the correct monotonic trend across the full range. Across both benchmarks, ABD is the only model that jointly achieves competitive quality at every individual scale and correctly captures the structural relationship between block size and generation difficulty.

\subsection{Zero-Shot Generalization}
Having established cross-scale robustness on in-distribution benchmarks, we ask whether ABD generalizes beyond its training distribution. Table~\ref{tab:zero_shot} reports zero-shot perplexity on Penn Treebank \citep{marcus-etal-1993-building}, Wikitext \citep{merity2018analysisneurallanguagemodeling}, LM1B, Lambada \citep{paperno2016lambadadatasetwordprediction}, AG News \citep{zhang2016characterlevelconvolutionalnetworkstext}, and Scientific Papers \citep{cohan2018discourseawareattentionmodelabstractive}, following the evaluation protocol of \citet{arriola2025blockdiffusioninterpolatingautoregressive}. We report ABD at BS=1 and BS=4, representing the quality ceiling and a moderately parallel operating point respectively.

\begin{table}[t]
\centering
\setlength{\tabcolsep}{4pt}
\caption{Zero-shot perplexity ($\downarrow$) on held-out benchmarks. All models are trained on 
OWT. Baseline numbers are taken from 
\citet{arriola2025blockdiffusioninterpolatingautoregressive}. \textbf{Bold}: best result per 
column, \underline{underline}: second best.}
\vspace{-5pt}
\label{tab:zero_shot}
\footnotesize
\begin{tabular}{lcccccccc}
\toprule
\textbf{Model} & \textbf{PTB} & \textbf{Wikitext} & \textbf{LM1B} & \textbf{Lambada} & \textbf{AG News} & \textbf{Pubmed} & \textbf{Arxiv} \\
\midrule
AR              & \textbf{81.07} & \textbf{25.32} & \textbf{51.14} & 52.13 & \textbf{52.11} & 48.59 & 41.22 \\
SEDD            & 96.33 & 35.98 & 68.14 & \underline{48.93} & 67.82 & 45.39 & 40.03 \\
MDLM            & 90.96 & 33.22 & 64.94 & \textbf{48.29} & 62.78 & 43.13 & \underline{37.89} \\
BD3LM (BS=4)    & 96.81 & 31.31 & 60.88 & 50.03 & 61.67 & \underline{42.52} & 39.20 \\
\midrule
ABD (BS=1, ours) & \underline{90.08} & \underline{29.69} & \underline{58.34} & 51.00 & \underline{61.25} & 43.50 & 39.11 \\
ABD (BS=4, ours) & 94.34 & 30.16 & 62.07 & 51.38 & 62.46 & \textbf{40.65} & \textbf{37.11} \\
\bottomrule
\end{tabular}
\vspace{-5pt}
\end{table}

ABD at BS=1 outperforms BD3LM across five of seven benchmarks, 
with marginal differences on the remaining two, despite BD3LM being evaluated 
at its optimal fixed block size. At BS=4, ABD remains competitive while enabling 
parallel generation, and notably outperforms BD3LM on scientific domains 
(Pubmed, Arxiv), suggesting that multi-scale training improves generalization 
to domain-shifted text. Taken together, these results confirm that ABD learns a configuration-invariant language representation rather than overfitting to a specific denoising structure, consistent with training over a broad support in $\Omega$.

\begin{table}[b]
\centering
\setlength{\tabcolsep}{5pt}
\caption{Generative perplexity (Gen. PPL $\downarrow$) and number of function 
evaluations (NFEs $\downarrow$) for 300 samples at sequence lengths $L = 1024, 2048$. All models are trained on OWT. Baseline numbers taken from 
\citet{arriola2025blockdiffusioninterpolatingautoregressive}. 
\textbf{Bold}: best diffusion result, \underline{underline}: second best diffusion result.}
\label{tab:gen_ppl}
\footnotesize
\begin{tabular}{llcccc}
\toprule
& & \multicolumn{2}{c}{$L = 1024$} & \multicolumn{2}{c}{$L = 2048$} \\
\cmidrule(lr){3-4} \cmidrule(lr){5-6}
\textbf{Category} & \textbf{Model} & \textbf{Gen. PPL} & \textbf{NFEs} & \textbf{Gen. PPL} & \textbf{NFEs} \\
\midrule
AR & AR & 14.1 & 1K & 13.2 & 2K \\
\midrule
\multirow{2}{*}{Diffusion} 
& SEDD & 52.0 & 1K & -- & -- \\
& MDLM & 46.8 & 1K & 41.3 & 2K \\
\midrule
\multirow{2}{*}{Block Diffusion}
& SSD-LM ($L'=25$) & 37.2 & 40K & 35.3 & 80K \\
& BD3LM ($L'=4$) & \underline{25.7} & 1K & \textbf{23.6} & 2K \\
\midrule
\multirow{1}{*}{ABD (ours)}
& BS=4 & \textbf{25.22} & 1K & \underline{24.46} & 2K \\
\bottomrule
\end{tabular}
\end{table}

\subsection{Generative Perplexity}

Table~\ref{tab:gen_ppl} reports generative perplexity and number of function 
evaluations (NFEs) for 300 samples at sequence lengths $L = 1024$ and $L = 2048$, 
following the evaluation protocol of \citet{arriola2025blockdiffusioninterpolatingautoregressive}. 
Generative perplexity is measured using GPT2-Large and reflects the 
quality of generated samples rather than the model's denoising ability, providing 
a complementary view to evaluation perplexity.

ABD matches the generative quality of the strongest fixed-block specialist at equivalent computational cost, despite never being trained at a fixed block size. This demonstrates that training over a broad configuration distribution does not degrade sample quality. Unlike BD3LM, which requires separate specialist checkpoints for each operating point, a single ABD model achieves competitive performance across all block sizes, reflecting its invariance to configuration.

\subsection{Ablations}
\label{sec:ablations}
All ablations are conducted on LM1B using the same architecture and finetuning 
protocol as the primary model unless otherwise stated.

\subsubsection{Effect of Configuration Distribution}
\label{sec:ablation_dist}
The ABD objective is defined with respect to a training distribution $\pi$ over 
the configuration space $\Omega$. We compare three choices: categorical 
exponential, categorical uniform, and categorical lognormal, each over block 
sizes $\{1, 2, 4, \dots, 128\}$. Figure~\ref{fig:ablation_dist} shows perplexity 
as a function of inference block size for all three distributions.

The results reveal a clean and interpretable pattern that follows directly from the support-weighting effects described in Section~4: the training distribution $\pi$ determines how denoising errors are weighted across configurations. The categorical exponential distribution, which concentrates mass on small block sizes, achieves the best perplexity at BS=1 at the cost of degraded quality at large block sizes. The categorical uniform distribution exhibits the opposite behavior, performing best at BS=128 while sacrificing small-block quality. The categorical lognormal distribution occupies a middle ground, performing competitively across the full range without excelling at either extreme. The choice of $\pi$ therefore provides a principled mechanism for practitioners 
to optimize for their target inference regime at training time, without any 
architectural changes.

\begin{figure}[t]
\centering
\includegraphics[width=\textwidth]{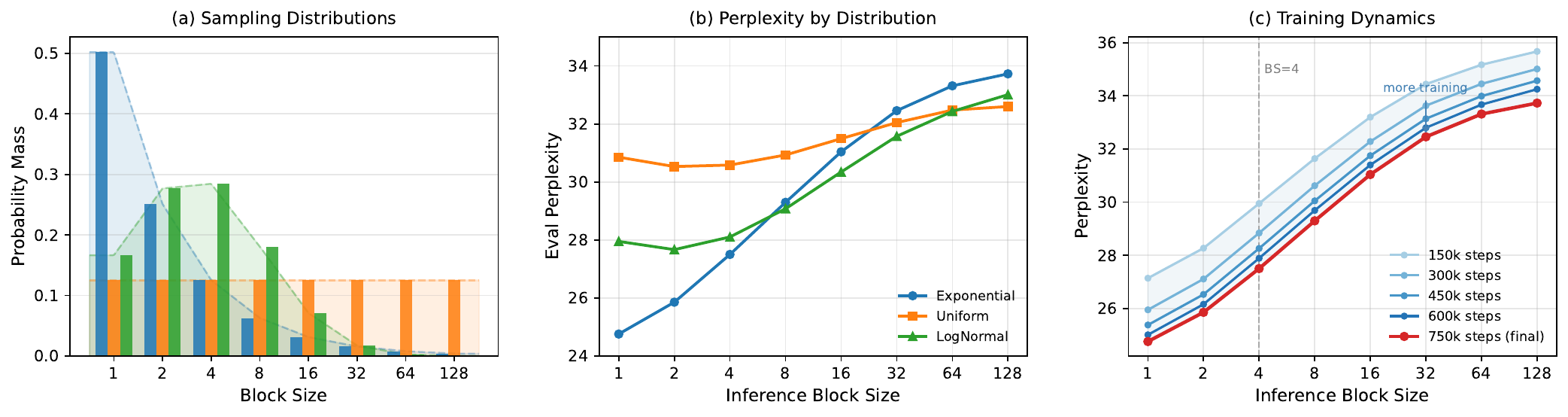}
\caption{Effect of block-size distribution $\pi$ and training dynamics on inference-time perplexity (LM1B, 750k steps from MDLM). \textbf{(a)} The three candidate distributions over block sizes. \textbf{(b)} $\pi$ controls the quality-granularity tradeoff: exponential weighting achieves lowest perplexity at small block sizes, lognormal balances the tradeoff, and uniform favors large blocks. \textbf{(c)} Perplexity improves across all block sizes throughout training, confirming multi-scale training acts as a stable regularizer.}
\label{fig:ablation_dist}
\end{figure}

\subsubsection{Training Dynamics}
\label{sec:ablation_dynamics}
Figure~\ref{fig:ablation_dist} shows eval perplexity as a function 
of training steps at various block sizes $\{1, 2, 4, \dots 128\}$ for the categorical exponential distribution based ABD model. 

The regularization effect of multi-scale training emerges early and strengthens 
throughout training: the model simultaneously improves across all block sizes 
rather than specializing toward any single scale. Perplexity at small block 
sizes improves most rapidly in early training, consistent with the exponential 
distribution placing higher probability mass on small block configurations. 
Across all block sizes, the model converges smoothly without instability, 
confirming that optimizing over a broad configuration distribution does not 
destabilize training dynamics.

\section{Discussion and Related Work}

\paragraph{Discrete Diffusion Language Models.}
Discrete diffusion for text generation was introduced by D3PM \citep{austin2023structureddenoisingdiffusionmodels}, which formalized denoising over discrete state spaces and established the masking-based formulation. Subsequent work improved performance and training efficiency, including SEDD \citep{lou2024discretediffusionmodelingestimating}, MDLM \citep{sahoo2024simpleeffectivemaskeddiffusion, shi2025simplifiedgeneralizedmaskeddiffusion}, and related approaches that improve training objectives and likelihood estimation. Structured extensions such as SSD-LM \citep{han2023ssdlmsemiautoregressivesimplexbaseddiffusion}, block diffusion models including BD3LM \citep{arriola2025blockdiffusioninterpolatingautoregressive}, and Eso-LMs \citep{sahoo2026esotericlanguagemodelsbridging} introduce prefix-conditioned generation by denoising token subsets, interpolating between autoregressive and fully parallel decoding. A number of these ideas have been scaled to larger model sizes \citep{nie2025largelanguagediffusionmodels,wu2025fastdllmv2efficientblockdiffusion, wang2025diffusionllmsfasterthanarinference}, and adapted commercially \citep{song2025seeddiffusionlargescalediffusion, labs2025mercuryultrafastlanguagemodels}. While these methods improve efficiency and controllability, they are typically trained over restricted subsets of the configuration space (e.g., fixed masking patterns or block-structured regimes), limiting their ability to generalize across decoding strategies without additional heuristics or architectural changes. In contrast, ABD treats the configuration as part of the learning problem, training over a distribution of configurations to enable a single model to generalize across decoding scales.

\paragraph{Diffusion Learning Dynamics and Inference-Time Techniques.}
Recent work has studied the scaling behavior, data efficiency, and learning dynamics of diffusion language models \citep{vonrutte2026scalingbehaviordiscretediffusion, ni2025diffusionlanguagemodelssuper, prabhudesai2025diffusionbeatsautoregressivedataconstrained}. These works provide insight into how diffusion models learn at the token and distribution level, but do not consider how training over restricted configuration subsets impacts generalization across decoding strategies. A complementary line of work focuses on improving performance through inference-time adaptation, including adaptive scheduling, semantic-aware decoding, and test-time scaling \citep{lu2026adablockdllmsemanticawarediffusionllm, lu2026advancingblockdiffusionlanguage, lee2025testtimescalingdiffusionllms, huang2025ctrldiffboostinglargediffusion}. Other approaches explore flexible generation orders or hybrid autoregressive--diffusion schemes \citep{kim2025anyorderflexiblelengthmasked, liu2025tidarthinkdiffusiontalk, kim2025trainworstplanbest}. In contrast, our work addresses this during training  by identifying configuration support as the key factor governing cross-scale generalization, enabling native support for arbitrary decoding strategies without test-time heuristics.

\paragraph{Limitations.} A key limitation of ABD is its dependence on the choice of configuration distribution $\pi$, which determines the tradeoff between performance across decoding regimes. While broad support over $\Omega$ enables generalization, suboptimal weighting can bias the model toward frequently sampled configurations, leading to uneven performance across scales. Additionally, ABD does not directly address inference efficiency: although it enables flexible decoding, selecting optimal inference-time policies remains an open problem. Finally, while our theoretical analysis characterizes optimality under support coverage, it does not provide finite-sample guarantees, and results may depend on training coverage quality.

\section{Conclusion}

We introduced Adaptive Block Diffusion (ABD), a training framework that resolves the mismatch between training and inference in diffusion language models by optimizing denoising risk over a distribution of configurations. By viewing diffusion training as risk minimization over configuration space, we showed that generalization across decoding strategies is governed by the support of the training distribution. This perspective explains the failure of fixed-structure approaches, which restrict training to a limited subset of configurations and degrade when evaluated off-grid.

Empirically, ABD demonstrates structural invariance across decoding scales, recovering the expected monotonic relationship between block size and perplexity while matching or outperforming fixed-block specialists. These results show that training over a broad configuration distribution enables robust generalization across inference regimes.

Overall, our work identifies configuration support as the central factor governing generalization in diffusion language models and establishes adaptive configuration training as a principled alternative to fixed-structure objectives.

\bibliography{colm2026_conference}
\bibliographystyle{colm2026_conference}
\appendix

\newpage
\section{Theoretical Proofs}

\subsection{Proof of Theorem~\ref{thm:abd-optimality}}

\begin{proof}
Let $(\Omega,\mathcal F,\pi)$ denote the probability space over configurations.
The functional
\[
\mathcal{R}_\pi(\theta)
=
\int_\Omega \mathcal{R}(\theta;k,\ell)\, d\pi(k,\ell)
\]
is well-defined and non-negative since
$\mathcal{R}(\theta;k,\ell)$ is a finite sum of Kullback--Leibler divergences.

By the information inequality, for any distributions $P,Q$,
$D_{\mathrm{KL}}(P\|Q) \ge 0$, with equality if and only if $P=Q$ almost
everywhere. Hence $\mathcal{R}(\theta;k,\ell) \ge 0$ for all $(k,\ell)$.

Let $\theta^\star$ be a global minimizer of $\mathcal{R}_\pi(\theta)$. Suppose
there exists a measurable set $A \subseteq \Omega$ with $\pi(A) > 0$ such that
\[
\mathcal{R}(\theta^\star;k,\ell)
>
\inf_{\theta \in \Theta} \mathcal{R}(\theta;k,\ell)
\quad \text{for all $(k,\ell)\in A$.}
\]
Then there exists $\varepsilon > 0$ such that
\[
\mathcal{R}(\theta^\star;k,\ell)
\ge
\inf_{\theta \in \Theta} \mathcal{R}(\theta;k,\ell) + \varepsilon
\quad \forall (k,\ell)\in A.
\]
Integrating over $\Omega$ yields
\[
\mathcal{R}_\pi(\theta^\star)
\ge
\int_\Omega
\inf_{\theta \in \Theta} \mathcal{R}(\theta;k,\ell)\, d\pi
+
\varepsilon\,\pi(A),
\]
which contradicts the optimality of $\theta^\star$. Therefore,
\[
\mathcal{R}(\theta^\star;k,\ell)
=
\inf_{\theta \in \Theta} \mathcal{R}(\theta;k,\ell)
\quad \text{for $\pi$-almost every $(k,\ell) \in \Omega$.}
\]
\end{proof}

\subsection{Proof of Theorem~\ref{thm:alignment}}

\begin{proof}
By the assumption $\Pi_{\mathrm{inf}} \ll \pi$, the Radon--Nikodym theorem
guarantees the existence of a non-negative measurable function
$f : \Omega \to \mathbb{R}_{\ge 0}$ such that for any measurable
$A \subseteq \Omega$,
\[
\Pi_{\mathrm{inf}}(A) = \int_A f(k,\ell)\, d\pi(k,\ell).
\]

Since $\mathcal{R}(\theta; k,\ell)$ is non-negative and measurable in
$(k,\ell)$, the inference risk can be written as
\[
\mathcal{R}_{\mathrm{inf}}(\theta)
=
\int_\Omega
f(k,\ell)\,\mathcal{R}(\theta; k,\ell)\, d\pi(k,\ell).
\]

Let $\theta^\star$ be a global minimizer of $\mathcal{R}_\pi(\theta)$. By
Theorem~\ref{thm:abd-optimality},
\[
\mathcal{R}(\theta^\star; k,\ell)
=
\inf_{\theta \in \Theta}
\mathcal{R}(\theta; k,\ell)
\quad \text{for $\pi$-almost every $(k,\ell)$.}
\]

Define the measurable set
\[
N := \{(k,\ell) \in \Omega \mid
\mathcal{R}(\theta^\star; k,\ell)
>
\inf_{\theta \in \Theta} \mathcal{R}(\theta; k,\ell)
\}.
\]
Then $\pi(N)=0$, and by absolute continuity, $\Pi_{\mathrm{inf}}(N)=0$.

For any $\theta \in \Theta$, we therefore have
\[
\mathcal{R}_{\mathrm{inf}}(\theta^\star)
=
\int_{\Omega \setminus N}
f(k,\ell)\,\mathcal{R}(\theta^\star; k,\ell)\, d\pi
\le
\int_{\Omega \setminus N}
f(k,\ell)\,\mathcal{R}(\theta; k,\ell)\, d\pi
=
\mathcal{R}_{\mathrm{inf}}(\theta),
\]
which shows that $\theta^\star$ is a global minimizer of
$\mathcal{R}_{\mathrm{inf}}(\theta)$.
\end{proof}

\subsection{Proof of Proposition~\ref{prop:block-limitation}}

\begin{proof}
Partition the configuration space as
\[
\Omega = \Omega_{\mathrm{BD}} \cup \Omega_{\mathrm{BD}}^{c}.
\]
By definition of $\pi_{\mathrm{BD}}$, for any measurable
$A \subseteq \Omega_{\mathrm{BD}}^{c}$ we have $\pi_{\mathrm{BD}}(A)=0$.

The expected training risk can therefore be written as
\[
\mathcal{R}_{\pi_{\mathrm{BD}}}(\theta)
=
\int_{\Omega_{\mathrm{BD}}}
\mathcal{R}(\theta; k,\ell)\, d\pi_{\mathrm{BD}}(k,\ell),
\]
since the contribution from $\Omega_{\mathrm{BD}}^{c}$ vanishes.

Thus, the value of $\mathcal{R}(\theta; k,\ell)$ for
$(k,\ell) \in \Omega_{\mathrm{BD}}^{c}$ does not affect the objective
$\mathcal{R}_{\pi_{\mathrm{BD}}}(\theta)$ and is unconstrained by its minimization.

Now suppose $\Pi_{\mathrm{inf}} \not\ll \pi_{\mathrm{BD}}$. Then there exists a
measurable set $A \subseteq \Omega$ such that $\pi_{\mathrm{BD}}(A)=0$ but
$\Pi_{\mathrm{inf}}(A)>0$. On this set $A$, the conditional risk
$\mathcal{R}(\theta^\star; k,\ell)$ is not optimized by training.

Since the inference risk
\[
\mathcal{R}_{\mathrm{inf}}(\theta^\star)
=
\int_{\Omega} \mathcal{R}(\theta^\star; k,\ell)\,
d\Pi_{\mathrm{inf}}(k,\ell)
\]
assigns positive mass to $A$, $\theta^\star$ is not guaranteed to minimize
$\mathcal{R}_{\mathrm{inf}}(\theta)$ over $\Theta$.
\end{proof}

\section{Implementation Details}
\label{app:implementation}

\subsection{Attention Mask Construction}
\label{app:mask}

The ABD attention mask is defined over a $2L$-length sequence formed by
concatenating the noisy tokens $\mathbf{z}_t \in \mathbb{R}^L$ and the clean
tokens $\mathbf{x}_0 \in \mathbb{R}^L$ as $[\mathbf{z}_t \mid \mathbf{x}_0]$.
Positions $0, \ldots, L{-}1$ index the noisy half and positions
$L, \ldots, 2L{-}1$ index the clean half.
Given a block assignment $\texttt{block\_ids} \in \mathbb{Z}^L$ that maps each
token to its block index, the mask is the union of three logical regions:

\begin{itemize}
    \item \textbf{Block-diagonal} ($M_\text{BD}$): a noisy token at position
          $i$ attends to noisy token $j$ if and only if they share the same
          block, i.e.\ $b_i = b_j$.  This restricts intra-half attention to
          within-block pairs.

    \item \textbf{Offset block-causal} ($M_\text{OBC}$): a noisy token at
          position $i$ attends to clean token $j$ if and only if its block
          strictly precedes the clean token's block, i.e.\ $b_i > b_j$.  This
          allows each block to condition on the clean prefix produced by all
          earlier blocks.

    \item \textbf{Block-causal} ($M_\text{BC}$): a clean token at position $i$
          attends to clean token $j$ if and only if $b_i \geq b_j$.  This
          enforces causal self-attention over the clean half at block
          granularity.
\end{itemize}

The quadrant $\mathbf{x}_0 \to \mathbf{z}_t$ is always zero, preventing
information flow from the clean half back into the noisy half.
Formally, the combined mask is

\begin{equation}
    M(i,j) \;=\; M_\text{BD}(i,j) \;\vee\; M_\text{OBC}(i,j) \;\vee\; M_\text{BC}(i,j).
    \label{eq:abd_mask}
\end{equation}

Figure~\ref{fig:attention_masks} contrasts the mask structure of standard
Block Diffusion (BD3LM), which uses fixed block boundaries, with that of
Adaptive Block Diffusion, where block boundaries are sampled stochastically
at each training step.  In the fixed-block case every block in $M_\text{BD}$
is identically sized; in the adaptive case block widths vary, reflecting the
sampled configuration distribution $\pi$.

\begin{figure}[h]
    \centering
    \includegraphics[width=\linewidth]{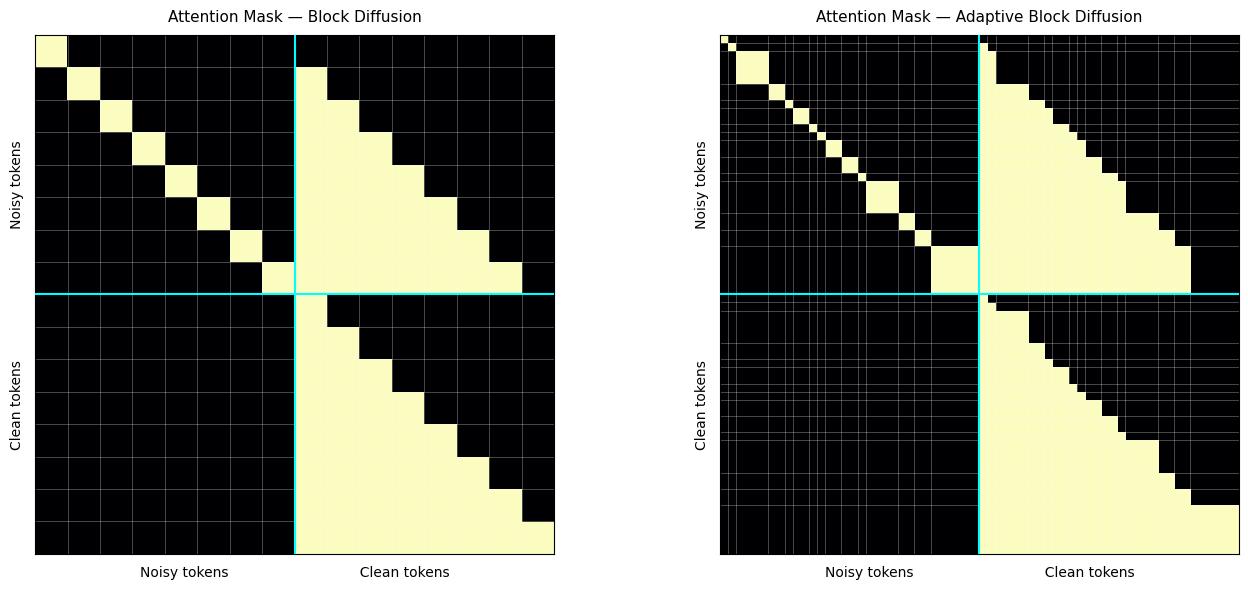}
    \caption{
        \textbf{Attention mask structure.}
        \emph{Left:} Block Diffusion (BD3LM) with a fixed block size.
        The noisy half (top-left quadrant) exhibits a uniform block-diagonal
        pattern; the clean half (bottom-right quadrant) is block-causal with
        equal-width steps.
        \emph{Right:} Adaptive Block Diffusion with a stochastically sampled
        configuration.  Block boundaries are irregular, reflecting a draw from
        the training distribution $\pi$ over prefix--window pairs
        $(k, \ell) \in \Omega$.  Both masks share the same three-region
        structure ($M_\text{BD}$, $M_\text{OBC}$, $M_\text{BC}$) defined in
        Eq.~\eqref{eq:abd_mask}; the quadrant $\mathbf{x}_0 \to \mathbf{z}_t$
        (bottom-left) is always masked.
        Cyan lines mark the noisy/clean boundary at position $L$.
    }
    \label{fig:attention_masks}
\end{figure}

The mask is implemented as a \texttt{flex\_attention} score modifier,
which avoids materialising the full $2L \times 2L$ boolean matrix:

\begin{lstlisting}[language=Python]
def abd_attention_mask(b, h, q_idx, kv_idx, *, block_ids, n):
    '''
    ABD score mask for flex_attention.

    The 2n-length sequence is laid out as [z_t (noisy) | x_0 (clean)].
    Positions 0..n-1 are noisy; positions n..2n-1 are clean.

    Args:
        b, h:      Batch and head indices (required by flex_attention API).
        q_idx:     Query position in [0, 2n).
        kv_idx:    Key-value position in [0, 2n).
        block_ids: LongTensor of shape (n,) mapping token i to its block index.
        n:         Original sequence length L.

    Returns:
        Boolean mask: True where attention is permitted.
    '''
    is_clean_q  = q_idx  >= n
    is_clean_kv = kv_idx >= n

    block_q  = block_ids[q_idx  % n]
    block_kv = block_ids[kv_idx % n]

    # M_BD: noisy -> noisy, same block
    m_bd  = (~is_clean_q) & (~is_clean_kv) & (block_q == block_kv)

    # M_OBC: noisy -> clean, strictly earlier block
    m_obc = (~is_clean_q) &   is_clean_kv  & (block_q >  block_kv)

    # M_BC: clean -> clean, current or earlier block
    m_bc  =   is_clean_q  &   is_clean_kv  & (block_q >= block_kv)

    return m_bd | m_obc | m_bc
\end{lstlisting}

\subsection{Configuration Distribution}
\label{app:config_dist}

Block configurations are sampled by \texttt{ABDBoundaryManager}, which draws
a single segmentation per training step and broadcasts it across the batch.
A sequence of block lengths is sampled from the chosen distribution $\pi$,
accumulated into boundaries via cumulative sum, and mapped to per-token block
indices by counting how many boundaries each token position exceeds.
The implementation supports six distributions over block lengths, all
parameterised by a target block size $\bar{\ell}$; the categorical family
additionally restricts support to powers of two
$\{2^0, 2^1, \ldots, 2^{\lfloor \log_2 L \rfloor}\}$:

\begin{lstlisting}[language=Python]
class ABDBoundaryManager:
    '''
    Samples stochastic block configurations for Adaptive Block Diffusion.

    A single segmentation is drawn per training step and broadcast across
    the batch, so all sequences share an identical block structure.

    Args:
        seq_len:     Sequence length L.
        target_size: Scale parameter for the expected block size.
    '''

    def __init__(self, seq_len: int, target_size: int = 16) -> None:
        self.L = seq_len
        self.target_size = target_size

    def get_block_ids(self, bsz, device, dist_type='categorical_exponential'):
        '''
        Sample block lengths from pi, accumulate boundaries, and return
        per-token block indices of shape (L,).
        '''
        lengths = self._sample_lengths(device, dist_type)   # (1, S)
        lengths = lengths.clamp(min=1, max=self.L).long()
        boundaries = lengths.cumsum(dim=-1)                 # (1, S)

        token_indices = torch.arange(self.L, device=device).unsqueeze(1)
        block_ids = (token_indices >= boundaries).sum(dim=-1)   # (L,)
        return block_ids

    def _sample_lengths(self, device, dist_type):
        S = (self.L // max(self.target_size // 2, 1)) + 5

        if dist_type == 'uniform':
            return torch.randint(
                1, 2 * self.target_size + 1, (1, S), device=device
            ).float()

        if dist_type == 'exponential':
            u = torch.rand((1, S), device=device).clamp(min=1e-7)
            return -torch.log(u) * self.target_size

        if dist_type == 'log_normal':
            mu = torch.log(torch.tensor(float(self.target_size), device=device))
            return torch.exp(torch.randn((1, S), device=device) * 0.5 + mu)

        # Categorical family: support restricted to powers of two
        max_power = int(torch.log2(torch.tensor(float(self.L))).item())
        sizes = 2 ** torch.arange(max_power + 1, device=device)
        S_cat = self.L

        if dist_type == 'categorical_exponential':
            # Exponential decay: P(2^k) proportional to 0.5^k
            probs = 0.5 ** torch.arange(len(sizes), device=device).float()

        elif dist_type == 'categorical_uniform':
            # Uniform over all powers of two
            probs = torch.ones(len(sizes), device=device)

        elif dist_type == 'categorical_log_normal':
            # Gaussian bump centred at target_size on the power-of-two grid
            mu = torch.log2(
                torch.tensor(float(self.target_size), device=device)
            )
            x = torch.arange(len(sizes), device=device).float()
            probs = torch.exp(-((x - mu) ** 2) / 2.0)

        else:
            raise ValueError(f'Unknown dist_type: {dist_type!r}')

        idx = torch.multinomial(probs, S_cat, replacement=True)
        return sizes[idx].unsqueeze(0).float()
\end{lstlisting}

\subsection{Experimental Details}
\label{app:experimental_details}

We use the same codebase, datasets, architecture, likelihood evaluation, 
and inference setup as \citet{arriola2025blockdiffusioninterpolatingautoregressive}, 
to which we refer the reader for full details. We describe only the aspects 
that differ.

\subsubsection{Training}
Our primary ABD model is initialized from a checkpoint trained with the 
full sequence length MDLM objective for 250k gradient steps, followed by 
finetuning for 750k steps under the categorical exponential configuration 
distribution. This differs from \citet{arriola2025blockdiffusioninterpolatingautoregressive}, 
which uses 850k MDLM pretraining steps followed by 150k block diffusion 
finetuning steps. The total training budget of 1000k steps is matched. 
We find that the diverse multi-scale objective of ABD requires a larger 
finetuning budget to converge across the full configuration space, 
motivating the reallocation toward ABD finetuning. Training translates 
to 65B tokens on LM1B and 524B tokens on OWT.

\subsubsection{Configuration Sampling}
At each training step, the best ABD model (presented in the main paper results) samples a block configuration $(k, \ell) \sim \pi$ from the categorical exponential distribution, which assigns 
probability proportional to $2^{-k}$ to the $k$-th block size in 
$\{1, 2, 4, \ldots, 128\}$, normalized to sum to 1. The attention mask 
is constructed dynamically based on the sampled configuration, as 
described in Appendix~\ref{app:implementation}.

\subsection{Design Choices: Distribution Type and Training Budget Allocation}
\label{app:design}

The design of ABD involves two coupled choices: the form of the configuration 
distribution $\pi$ and the allocation of the training budget between MDLM 
pretraining and ABD finetuning. We describe the reasoning behind these choices 
and provide supporting evidence.

\paragraph{Training budget allocation.} Following BD3LM \citep{arriola2025blockdiffusioninterpolatingautoregressive}, 
we initially allocated 850k steps to MDLM pretraining and 150k steps to ABD 
finetuning. However, the diversity of the ABD training objective, which 
optimizes over a broad distribution of configurations rather than a single 
fixed block size, requires more ABD training steps to converge. With only 
150k ABD steps, the model retains strong MDLM-style performance at large block 
sizes but underfits at small block sizes. Training from scratch for 1000k ABD 
steps avoids this underfitting but produces the opposite failure: without MDLM 
initialization, the model lacks a coherent language modeling foundation and 
collapses at large block sizes, as shown in Table~\ref{tab:init_continuous}. 
The continuous uniform scratch model reaches 142.69 perplexity at BS=128, 
worse than MDLM at BS=1 (106.56). We found that 250k steps of MDLM pretraining 
followed by 750k steps of ABD finetuning achieves the best balance: sufficient 
language modeling grounding combined with enough ABD training to converge across 
the full configuration space.

\paragraph{Discrete vs.\ continuous distributions.} We initially trained with 
continuous distributions over block sizes (exponential, uniform, lognormal), 
which provide dense coverage of the configuration space $\Omega$. However, 
continuous distributions make optimization difficult: the model must simultaneously 
learn to denoise at arbitrarily sized windows, leading to slower convergence and 
unstable performance at extreme block sizes. Restricting $\pi$ to powers of two 
$\{1, 2, 4, \ldots, 128\}$ preserves broad coverage of the parallelism spectrum 
while making the optimization problem more tractable. Each power-of-two block 
size represents a qualitatively distinct operating point from fully 
autoregressive ($\ell=1$) to highly parallel ($\ell=128$), and the discrete 
grid avoids the redundancy of nearby continuous configurations that contribute 
similar gradients. Table~\ref{tab:disc_vs_cont} shows that discrete distributions 
consistently outperform their continuous counterparts under matched total training 
budgets, particularly at small block sizes where continuous distributions 
underperform despite nominally covering those configurations.

\begin{table}[t]
\centering
\setlength{\tabcolsep}{4pt}
\caption{Perplexity on LM1B comparing discrete categorical vs.\ continuous 
configuration distributions. All models are trained for 1000k total steps: 
discrete models use 250k MDLM + 750k ABD; continuous models use 850k MDLM 
+ 150k ABD. The key variable is the proportion of the budget allocated to 
the ABD objective.}
\label{tab:disc_vs_cont}
\scriptsize
\begin{tabular}{llccccccc}
\toprule
\textbf{Type} & \textbf{Distribution} & \textbf{1} & 
\textbf{4} & \textbf{8} & \textbf{16} & \textbf{32} & \textbf{64} & 
\textbf{128} \\
\midrule
\multirow{3}{*}{Discrete (250k+750k)}
& Exponential & \textbf{24.76} & \textbf{27.50} & \textbf{29.30} & \textbf{31.04} & 32.46 & 33.31 & 33.73 \\
& Uniform     & 30.86  & 30.58 & 30.93 & 31.49 & {32.05} & {32.48} & {32.60} \\
& LogNormal   & 27.95 & 28.10 & 29.08 & 30.34 & \textbf{31.58} & 32.43 & 33.01 \\
\midrule
\multirow{2}{*}{Continuous (850k+150k)}
& Exponential & 30.84  & 30.22 & 30.59 & 31.10 & 31.64 & \textbf{32.04} & \textbf{32.29} \\
& Uniform     & 31.41  & 30.27 & 30.54 & 31.04 & 31.61 & 32.07 & 32.39 \\
& LogNormal   & 32.27 & 30.43 & 30.56 & 31.07 & 31.71 & 32.12 & 32.47 \\
\bottomrule
\end{tabular}
\end{table}

\begin{table}[t]
\centering
\setlength{\tabcolsep}{4pt}
\caption{Perplexity on LM1B: effect of training budget allocation for 
continuous distributions. All models trained for 1000k total steps. 
Scratch models allocate all steps to ABD; finetune models use 850k MDLM 
+ 150k ABD. MDLM initialization prevents collapse at large block sizes 
but requires sufficient ABD training to generalize at small block sizes.}
\label{tab:init_continuous}
\scriptsize
\begin{tabular}{llccccccc}
\toprule
\textbf{Init} & \textbf{Distribution} & \textbf{1} & 
\textbf{4} & \textbf{8} & \textbf{16} & \textbf{32} & \textbf{64} & 
\textbf{128} \\
\midrule
\multirow{3}{*}{Scratch (1000k ABD)}
& Exponential & \textbf{28.50}  & 28.95 & 29.77 & 30.68 & 32.70 & 36.88 & 43.05 \\
& Uniform     & 28.93  & \textbf{28.82} & \textbf{29.48} & \textbf{30.33} & 31.83 & 44.08 & 142.69 \\
& LogNormal   & 30.62  & 29.02 & 29.62 & 30.48 & 32.54 & 34.82 & 38.30 \\
\midrule
\multirow{2}{*}{Finetune (850k+150k)}
& Exponential & 30.84 & 30.22 & 30.59 & 31.10 & 31.64 & \textbf{32.04} & \textbf{32.29} \\
& Uniform     & 31.41 & 30.27 & 30.54 & 31.04 & \textbf{31.61} & 32.07 & 32.39 \\
& LogNormal   & 32.27 & 30.43 & 30.56 & 31.07 & 31.71 & 32.12 & 32.47 \\
\bottomrule
\end{tabular}
\end{table}

Together, these results motivate the final ABD design: discrete categorical 
distributions over powers of two, with 250k steps of MDLM pretraining and 
750k steps of ABD finetuning. This combination provides sufficient language 
modeling grounding, tractable optimization over the configuration space, and 
enough ABD training to generalize across the full block size range.

\section{Samples}
\begin{figure}[t]
\centering
\label{fig:sample1k}
\begin{tcolorbox}
\scriptsize
Top 2 Most Powerful Internet Utilities in 2016 – How Does your Experience Affect your View of the Internet? While reading about E.T. smart meters, I thought “So the answer to all of the utilities using smart meters to monitor and analyze their network is: Put a smart meters in your closet.” So, the answer to all of the utilities using smart meters to monitor and analyze their network is: Put a smart meters in your closet. E.T. smart meters can be used for meters for your Internet access or for mobile usage. Like your home internet phone and netboard, Smart meters help users, businesses, and organizations get connected devices to their site with the Internet. Smart meters enable a lot better online education, marketing, and purchasing. But in most cases, you can buy a smart meter with your best Internet usage statistics and report all the areas covered in the meter. As we mentioned earlier, smart meter is just a form of forlier meter which that uses a unique identification and generates data by observing the Internet usage of a smart device. There are so many companies that claim to provide smart meters that in no way way relies on what you do. You are not relying on what you buy, you’re going to rely on what you do. It is a big technical problem to implement smart meters. However, how do they compare to different utilities? Even if you still use various smart utilities, you need to check out the experience of the utility? What is your Internet usage? What are the benefits in terms of data use? Somehow should your utility analyze data to make that decision?Since the Internet is a really interesting endeavor, it has been a real challenge to find the right fit for your utility and to analyze data there.How often does an outage affect your overall view of the Internet? The answer isn’t that often. Sometimes a large outage will affect your Internet usage but it also won’t affect your overall view of the Internet. For example, during a long outage, does it affect your overall view of the Internet? What does your data consumption mean to your view of the Internet or the utility? Most of the time, there’s no data you can take away from you until something happens. What would it mean if the global temperature was at 34 °C? What are your recommendations for achieving the highest levels of reliability on your utility? What if the power doesn’t go through the breakdown? What would it mean to increase your view of the Internet or your reliability? What is it about Internet usage that you want to improve more? Looking at your utility’s Internet monitoring reports, there are a number of factors that determine how the utility impacts your own view of the Internet. The three main factors vary between utility systems so there is a chance that the data on your utility may change or not have always been used. How should you determine the data on your utility? When you’re getting more and more data from your utility, you need to trust that your utility is not violating any privacy rules. As much as you want to do it, you want to be aware of what you’re doing and do not trust your utility to collect data on you any more than you want to monitor your website. At the end of the day, you use your utility settings to make a choice for you, and once you’ve used all your settings, you probably should review it and see what data it collects. Most of all, you need to make sure your utility is the right one for your needs. Not only do you need to figure out how you’re going to maximize your data use, but you need to find out about your utility and its data consumption. What your own utility has done to improve their utility over the years? In 2017, we highlighted some examples of how a utility utility has continued to improve, or enhance, their utility data. FireMapp offers a feature to share with your users their location usage. This feature has been shown to have played a big role in improving this site’s performance. SuperBuddy sells the ability to share your data usage within a building. Another example is that of ADAPT, the app that makes it easy for companies to share data. They do this by checking factors like “home minutes spent”, “work time spent, “pea time spent, “drime and cost”, “price ratio”, “cost”, “speed”, “hours”, etc. SmartGoOn provides a tool for sharing your data: you can see how you personal usage compares to your utility’s data usage. The features allow you to easily compare your different habits, by comparing your usage
\end{tcolorbox}
\caption{Sample from ABD for block size L' = 4 of length L = 1023 under T = 5K diffusion steps (trained with a context length of L = 1024). The generative perplexity of this sample under GPT2-Large is 15.3, and its entropy is 5.03.}
\end{figure}

\begin{figure}[t]
\vspace{-30pt}
\centering
\label{fig:sample2k}
\begin{tcolorbox}
\scriptsize
British Prime Minister David Cameron has mocked Scotland, Wales and England in his ""surprise talk of the day"" as he lined up a spot for London\'s Olympics Summer Games. Cameron played down criticism the case raised by the USA – the city where so many of its residents live, he said – could host the Summer Games. ""These in England are no more amazing - or thrilling – than the Olympic Games in Beijing. It is going to be all the more fun to see them,"" Cameron told the London Evening Standard. ""And I\'d like to offer some favours to those responsible for bringing those Games across the lines. ""Now it is great for the people in England to not be the only ones us are thinking about being a special guest. ""A handful of cabinet ministers can come. That is the huge bonus. That\'s why it is my pleasure to put them on."" Cameron\'s remark comes as the Summer Games begin at the start of this month, scheduled for more than £30m. Games are expected to showcase some of Britain’s best and brightest talent. In a measure of these sights, food has been offered in an area dominated by few Britons. On Thursday, the first group of chefs in Olympic Park were signing promotional contracts in silver bottles of Champagne. Later that day, the Quiet House chef known as Eastie nominated a piece of a cocktail to be the highlight of the two-day line of shows. In a bid to make a promise to all the children at the Olympics, he snapped a date of the Royal Pavilion - home to British-owned Brixton - and delivered the cocktail to the children’s playground in the Thames for several days. Other champagne favourites Chef Ray Dussosts of Corry Street, London, and Ryan Haynes of Jeref Seafood in Londonderry have also won guest spots. Joining Gillon, his stepfather and sister, are the two white boys running the Road Cocktail Kitchen from David, the Castle Tavern in Devon, and Tyrone McCray of the Wild Firestone Café in Manchester. A Bevan Beard chef by profession, Harry Harrison, will spend his \$125,000 cocktail cash at a public restaurant ice-cream party next month. Additionally, three men from Southampton will also join Ian Searls of Fulham, Kent, and Mark Frost of Zappros in presenting their own cocktail at a private event in the northern city. Cameron is expected to eat at this summer’s world championships and the event is expected to be watched by the millions of spectators in London, Manchester and other cities around the world. England and Wales’ respective governments, the UK and the Commonwealth, have chosen not to host the games.<|endoftext|>The long-awaited Free Parking Report was finally unveiled at ""Mainline,"" Boston\'s news forum for downtown. That said, is he looking at the worst and the worst transit areas where parking is unsafe? Criments The long-awaited Free Parking Report was finally unveiled at ""Mainline,"" Boston\'s news forum for downtown. That said, is he looking at the worst and the worst transit areas where parking is unsafe? Many people are aware of the study\'s findings about how the extension to Logan Square is affecting transit accessibility, but the official recommendations on surface and subway stations were not made in a complete copy of the report. Mayor and his current administration, Mayor Marty Walsh, took a rough assessment of the preliminary study, as did the findings of the Secretary of Transportation. Now that that assessment has been broken into four parts, starting with the upper-level, like city Park South, by the time the transit department ramps up the installation of the sodas next year, the recommendations will cover the final four sections. Making a stab at business One thing to stand out from the official report is the new report\'s finding that Bostonians want a better transit system, and therefore, are concerned that certain parts of the city should be transferred to light buses rather than on light cars. ""I generally think of light buses as inefficient, in that sense. A bus is parked, with an elevator system, with people trying to get off the bus. I think it\'s a significant disadvantage to someone who can\'t afford to live in the area. I think there are more opportunities for transportation. So is the use of buses? Are people in an area where buses provide a way to go, that they look to work rather than ride the trolleybus, that is similar to what is going on in the mall or the city?"" one member asked. Trolley or tram At a meeting Monday morning, Mayor Walsh noted a concern about the final recommendation\'s focus on a ""light bus"" rather than a transit system. The report doesn\'t do that. Instead, it said it should be developed as rail-style buses. The report focuses on light light buses for riders, who are the only carless option on the network. Light Rail buses would operate on an underground line, for example. If enough people would have enough electricity to ride on a train at the line\'s stations, that would help reduce the distance to ride on a bus. Under the plan, the train route would also operate slowly, which would encourage use of the more expensive light buses. Bus Rapid transit Boston officials are actively discussing the development of a bus rapid transit, a rapid line that tracks passengers at or near bus stops, to run sometime in the future. One of the buses will be Link - the LRT line running between Boston Common and Logan Square and is projected to be out in 2015. But the study was conducted only on the MBTA, not bus rapid transit or its regional transit plans, the only map released to the public available at these locations. The light rail line would have about 120 stops, while the seven buses at each station at city Park had 30. That required a bus to provide multiple stops on three stations. As a result, the line would generate about 150,000 trips a year. Advocates want MBTA to move quickly But one of the people enthusiastic about the new light rail was a local developer who is pushing the idea. Morrow from Fairmont Plaza developer Robert Cummins, which proposed 15 years ago to pay for a bike-friendly version of the line, the new light rail project could not have been prepared for the realities of the city. In part, the project will rely heavily on maintenance. In the meantime, city officials are lining up financing money from public sources, two of them from Boston Common. Five years ago, Mayor Walsh pointed out that the city hasn\'t built a new light rail since 1981, when some investors along with the state legislature approved the Medford Line. A \$390 million master plan from Massachusetts Department of Transportation would pay to replace the Medford Line, too. The MBTA has also been partnering with the American College of Bicy Cities to provide an alternate route to work on at Fairmont Plaza. The study found that a light rail route would eliminate the need for additional four stops along the route, improving bus service and allowing more people to use the system. Freightrail plan There\'s detail to the report doesn\'t cover the cost estimates, but one of the findings is that the light rail line will provide short-term employment. But it will create a considerable spur from the city of downtown Boston that has not developed very large enough. Transit Alternatives Communications, a statewide advocacy group over the past decade, has offered this economic theory in a detailed analysis of recent light rail development across Boston, and it hasn\'t been borne out by anyone. Other studies, however, don\'t provide enough details to support that conclusion. Fellow cities such as Long Beach, Athens-Oxford, Oakland, Detroit, and others with greater transit needs are developing light rail options that are useful but not nearly as impressive as the Portland rail-rail network, an alternative to light rail in Portland with many concerns. Changes made Even if the government does do enough to keep the Gold Line on the ground, it will still be agonizing to pull that plug when the official report is in place. One of the recommendations from the consultant suggested that a TTA board, at least those close to the MBTA system, consider a toll-free bikeway along the extension of the Gold Line line rather than light rail. And when considering a bikeway at the southern end of the line that would connect the Gold Line to the Extension of the Green Line, the consultant recommended that it would be safe to go the other way. That proposal, in short, might be a new proposal. But at the moment, the only deal on the Gold Line extension is the one on the bridge on the Westmead/Hudson/Green Line extension.<|endoftext|>The economic crisis over whether we should bring down our debt triggered a government shutdown back in September when House Speaker John Boehner (R-Ohio) attacked Democrats across party lines for refusing to reopen the government, just hours after the shutdown had already been declared. But on Tuesday, Republicans will not be giving up that fight, and it is this fight that is going to affect the party’s prospects for the House and Senate. The GOP says they are prepared to offer measures to avoid a repeat of the House shutdown. But they are hesitant to do so, and Republicans are already feeling that pressure. ""You think that the last economic downturn of this kind is here and you think the last in history is here and you think the last the entire last century and you can\'t afford to be the same before that it is,” said Rep. Pete Sessions (R-Texas
\end{tcolorbox}
\caption{Sample from ABD for block size L' = 4 of length L = 2047 under T = 5K diffusion steps (trained with a context length of L = 1024). The generative perplexity of this sample under GPT2-Large is 25.68, and its entropy is 5.65.}

\end{figure}
\end{document}